\documentclass[a4paper]{article}
\usepackage[a4paper, margin=2.5cm]{geometry}

\usepackage[english]{babel}
\usepackage{csquotes}
\usepackage[toc,page]{appendix}

\usepackage{upgreek}
\usepackage{amsmath}
\usepackage{mathrsfs}
\usepackage{mathtools}
\usepackage{amssymb}
\usepackage{enumerate}
\usepackage{dsfont}
\usepackage{bm}
\usepackage{hyperref}
\usepackage{physics}
\usepackage{multicol}
\usepackage{xcolor}
\usepackage{centernot}
\usepackage{parskip}
\usepackage{authblk}

\DeclarePairedDelimiterX\Set[1]\lbrace\rbrace{\def\given{\;\delimsize\vert\;\allowbreak}#1}

\usepackage{amsthm}

\newtheorem{theorem}{\textbf{Theorem}}

\newtheorem{proposition}{\textbf{Proposition}}[section]
\newtheorem{corollary}{\textbf{Corollary}}[section]
\newtheorem{lemma}{\textbf{Lemma}}[section]
\newtheorem*{remark}{\textbf{Remark}}

\usepackage{hyperref}
\usepackage{cleveref}
\usepackage{crossreftools}

\numberwithin{equation}{section}

\DeclareMathOperator*{\argmin}{arg\,min}

\DeclareMathOperator{\repu}{RePU}
\DeclareMathOperator{\relu}{ReLU}
\DeclareMathOperator{\step}{Step}

\DeclareMathOperator{\sawtooth}{SawTooth}
\let\prime\relax
\DeclareMathOperator{\prime}{prime}
\DeclareMathOperator{\bF}{Beta}

\newcommand{\G}{\mathbb{G}}
\newcommand{\R}{\mathbb{R}}

\newcommand{\N}{\mathbb{N}}
\newcommand{\X}{\mathcal{X}}
\newcommand{\B}{\mathcal{B}}
\newcommand{\M}{\mathcal{M}}

\renewcommand{\S}{\mathbb{S}}
\newcommand{\FS}{\mathscr{F}}
\newcommand{\e}{\mathrm{e}}

\allowdisplaybreaks

\usepackage{lastpage}

\usepackage{fancyhdr}
\allowdisplaybreaks

\usepackage[
    backend=biber,natbib,
    style=apa
]{biblatex}

\title{Embeddings between Barron spaces with higher order activation functions}
\author[1,*]{Tjeerd Jan Heeringa}
\author[1]{Len Spek}
\author[2]{Felix Schwenninger}
\author[1]{Christoph Brune}
\affil[1]{Mathematics of Imaging \& AI, University of Twente, Enschede, The Netherlands}
\affil[2]{Mathematics of Systems Theory, University of Twente, Enschede, The Netherlands}
\affil[*]{Corresponding author: t.j.heeringa@utwente.nl}
\date{2024-05-24}

\begin{document}

\maketitle
\thispagestyle{empty}

\begin{abstract}
    \noindent The approximation properties of infinitely wide shallow neural networks heavily depend on the choice of the activation function. To understand this influence, we study embeddings between Barron spaces with different activation functions. These embeddings are proven by providing push-forward maps on the measures $\mu$ used to represent functions $f$. An activation function of particular interest is the rectified power unit (RePU) given by $\repu_s(x)=\max(0,x)^s$. For many commonly used activation functions, the well-known Taylor remainder theorem can be used to construct a push-forward map, which allows us to prove the embedding of the associated Barron space into a Barron space with a $\repu$ as activation function. Moreover, the Barron spaces associated with the $\repu_s$ have a hierarchical structure similar to the Sobolev spaces $H^{s}$.
    \phantom{A}
    \\ \\
    \textbf{keywords: }Neural networks, activation functions, Barron spaces, push-forward, rectified power unit, Theory of Machine Learning
\end{abstract}

\cfoot{Page \thepage \hspace{1pt} of \pageref{LastPage}}

\section{Introduction}
For a given function $\sigma:\R\to\R$, called the activation function, we consider the set of functions that can be written using an integral representation as
\begin{equation}\label{eq:integral_formulation}
    f(x) = \int_\Omega \sigma(\braket{x}{w}+b)d\mu(w,b), \quad x\in \X:=[-1,1]^d
\end{equation}
where $\mu$ is a Radon measure on $\Omega\subseteq\R^{d+1}$. In machine learning, this represents an infinitely wide shallow neural network.

The choice of activation function is a crucial aspect in the design and training of neural networks, as it directly affects their expressivity and generalizability. The search for new activation functions that can better capture the underlying structure and patterns of the data, while avoiding common issues such as vanishing gradients and overfitting, is ongoing. The most popular choices for the activation function are the sigmoidal functions\footnote{Sigmoidal functions are monotonic increasing functions which go to constants at plus and minus infinity.} and the $\relu(x):=\max(0,x)$. Although these activation functions work in many instances, they are not the best for all instances. For example, \cite{siegel_high-order_2021} showed that in certain instances taking a higher-order version of the $\relu$ given by $\repu_s(x):=\max(0,x)^s$ yields improved approximation properties compared to using $\relu$. Similarly, \cite{hendrycks_gaussian_2020};\cite{ramachandran_searching_2023};\cite{misra_mish_2020} showed that taking smoothed versions of $\relu$ like GeLU, Swish or Mish as activation functions also accomplished this. Towards better approximation properties, we want to understand how the functions spaces associated with neural networks change when the activation function is changed.

Changing the activation function means that the associated vector space, and in particular its natural norm, gets changed too. The natural norm of the space for deep neural networks is unknown. For shallow neural networks, the Barron norm gives the natural norm for an infinitely wide neural network with activation function $\sigma$. For sigmoidal activation functions, this is given by 
\begin{equation}
    \norm{f}_{\B_{\sigma}} = \inf \int_\Omega(1+\norm{w}_{\ell^1}+\abs{b})d\abs{\mu}(w,b),
\end{equation}
where the infimum is taken over all measures $\mu$ such that \eqref{eq:integral_formulation} holds for all $x\in\X$. A vector space with this norm is called a \textit{Barron space} (associated with the activation function $\sigma$)[\cite{e_observations_2022}]. Hence, to understand what effect a change to the activation function has, we should determine what effect this change has on the underlying Barron space.

\subsection{Related work}
The Barron spaces were introduced with the motivation to create a \enquote{reasonably simple and transparent framework for machine learning}[Page 2 of \cite{e_observations_2022}]. Initially, they were introduced only for the $\relu$ and sigmoidal activation functions. In [\cite{bartolucci_understanding_2023}], Barron spaces were shown to be \textit{reproducing kernel Banach spaces} (RKBS), a Banach space analogue to \textit{reproducing kernel Hilbert spaces} (RKHS). This was used to determine what form the Barron norm should have for the \textit{rectified power unit} ($\repu$), but can easily be extended to determine what a natural norm would be for any activation function. It was proven that Barron functions have bounded point evaluations [\cite{bartolucci_understanding_2023}; \cite{spek_duality_2023}](due to them being an RKBS), Barron functions can be approximated in $L^p$ with rate $O(m^{-1/2})$ [\cite{e_towards_2020},\cite{siegel_sharp_2022}], Barron spaces have a representer theorem[\cite{parhi_banach_2021}] and more. Some of these results hold for particular activation functions (mostly $\relu$), whereas others hold for more general classes of activation functions. 

In order to extend results for the Barron space with $\relu$ as the activation function to more general activation functions, a relation between $\relu$ and a large class of activation functions was established in [\cite{li_complexity_2020}]. They determined that any activation function $\phi$ that satisfies
\begin{equation}\label{eq:li_et_al}
    \int_\R \abs{\partial^2\phi(x)}(1+\abs{x})dx <\infty
\end{equation}
can be approximated up to arbitrary precision in $L^\infty$ by a finite linear combination of $\relu$s. This covers, among others, the sigmoidal activation functions. Their result does not apply directly to the infinite width setting of the Barron spaces, but their strategy allows for an extension to the infinite width setting. 

In [\cite{caragea_neural_2020}], the relations between the Barron spaces and the related spectral Barron spaces were discussed. For $s\in\N$ the spectral Barron spaces have norm
\begin{equation}
    \norm{f}_{\B_{\FS,s}} = \inf\int_{\R^d}(1+\norm{\xi}_{\ell^1})^s\abs{\hat{f_e}(\xi)}d\xi
\end{equation}
where the infimum is taken over all extensions $f_e\in L^1(\R^d)$ of $f$. They can be seen as Barron spaces with a cosine as the activation function. It was shown that these spaces are closely related to but distinct from the Barron spaces with ReLU as the activation function. In particular, $s\geq2$ needs to hold to have an embedding into the Barron space with $\relu$ as the activation function. They also show that the Barron spaces with ReLU as the activation function embeds into that with Heaviside step function $\step$ as activation function.

In [\cite{siegel_approximation_2020}], it was shown that functions $f\in\B_{\FS,s+1}$ can be approximated in $H^s$ with rate $O(m^{-0.5})$ when using activation functions $\sigma\in W^{s,\infty}_{loc}$, with which a finite linear combination can be formed that decays sufficiently fast. Their work does not provide an embedding between the respective spaces. 

In practice, there are many more activation functions being used. ReLU6 is a linear combination of $\relu$s [\cite{howard_mobilenets_2017};\cite{maas_rectifier_2013}]. Tanh is a convolution of ArcTan with a specific kernel. SoftPlus and $\relu$ are derivatives of Sigmoid and squared $\relu$ respectively [\cite{glorot_deep_2011}]. HardSwish, SILU/Swish-1, GeLU and the growing cosine unit are ReLU6, Sigmoid, Gaussian normal CDF function, and cosine respectively multiplied by their input [\cite{howard_searching_2019};\cite{ramachandran_searching_2023};\cite{hendrycks_gaussian_2020};\cite{noel_growing_2021}]. What these and other changes to the activation do to the associated Barron space is unknown. 

\subsection{Our contribution}
In this work, we show that many of these changes to the activation function lead to an embedding between the respective Barron spaces. The main idea is that we explicitly construct a push-forward map $\Theta$ for two given activation functions $\sigma$ and $\phi$ so that
\begin{equation}
    f(x) = \int_\Omega \sigma(\braket{x}{w}+b)d\mu(w,b) = \int_\Omega \phi(\braket{x}{w}+b)d\Theta_{\#}\mu(w,b), \quad x\in \X.
\end{equation}
When we find a map $\Theta$, we can use it to show that the Barron norm $\norm{f}_{\B_{\phi}}$ is finite and determine the constant of embedding by using the relation $\Theta$ induces between $\Theta_{\#}\mu$ and $\mu$. For example, for two non-ReLU Lipschitz continuous activation functions $\sigma$ and $\phi$ we need $\Theta$ to be such that
\begin{equation}\label{eq:push_forward_example}
    \norm{f}_{\B_{\phi}} \leq \int_\Omega 1+\norm{w}_{\ell^1}+\abs{b}d\abs{\Theta_{\#}\mu}(w,b) \lesssim \int_\Omega 1+\norm{w}_{\ell^1}+\abs{b}d\abs{\mu}(w,b)
\end{equation}
for all $\mu$ satisfying \eqref{eq:integral_formulation} to get an embedding. The embeddings can be grouped into those in which at least one of the two activation functions is a $\repu_s$ for some $s\geq 0$ and those in which neither activation function is a $\repu_s$. The former is discussed in \Cref{sec:embeddings}, whilst the latter is discussed in \Cref{sec:generic}. Additionally, in \Cref{sec:taylor} we show how the push-forward strategy can be used to provide embeddings from non-Barron spaces into Barron spaces by proving the embedding of the spectral Barron spaces into the Barron spaces with $\repu$ activation. The proven embeddings are summarized in \Cref{th:new}.
\begin{theorem}\label{th:new}
Let $\psi$ and $\phi$ be Lipschitz activation functions. If
\begin{equation}\label{eq:th_new}
    \phi(x) = \int_{\R^2}\psi(xw +b)d\gamma(w,b)
\end{equation}
for all $x\in\R$ and for some measure $\gamma\in \M(\R^2)$ satisfying
\begin{equation}
    \int_{\R^2}(1+\abs{w}+\abs{b})d\abs{\gamma}(w,b)<\infty,
\end{equation}
then $\B_\phi\hookrightarrow\B_{\psi}$. \\
Moreover, 
\begin{enumerate}
    \item $\B_\psi\hookrightarrow\B_{\repu_1}$ whenever $\psi\in C^1(\R)$ with $\partial^2\psi\in L^1(\R)$,
    \item $\B_\psi\hookrightarrow\B_{\repu_s}$ whenever $s\geq 1$, $\psi\in C^{\lceil s\rceil}(\R)$ with the right Caputo derivative $\partial^{s+1}_+\psi\in L^1((0,\infty)$ and the left Caputo derivative $\partial^{s+1}_-\psi\in L^1((-\infty,0)$, and $\Omega$ is bounded,
    \item $\B_{\repu_t}\hookrightarrow\B_{\repu_s}$ for $0\leq s \leq t$,
    \item $\B_{\FS,m+1}\hookrightarrow\B_{\repu_m}$ for all $m\in\N$.
\end{enumerate}
\end{theorem}
The main statements shows that embeddings between Barron spaces with Lipschitz activation functions can be determined by only looking at the relation between their activation functions. In the four points below the main statement, the first two are generalizations of \cite{li_complexity_2020}. The remaining two points are generalizations of \cite{caragea_neural_2020}. In particular, the first point relaxes condition \eqref{eq:li_et_al} from \cite{li_complexity_2020} to $\partial^2\psi\in L^1(\R)$. The second point extends the result to higher smoothness at the cost of requiring $\Omega$ to be bounded. The third point refines the embedding result $\B_{\repu_1}\hookrightarrow\B_{\repu_0}$ of Lemma 7.1 1) from \cite{caragea_neural_2020} to one covering all $\repu_s$ with $s\geq 0$. The fourth point generalizes the embedding $\B_{\FS,m+1}\hookrightarrow\B_{\repu_m}$ of Lemma 7.1 3) from \cite{caragea_neural_2020} for $m=1$ to all integers $m\in \N$.

Although the first two points of \Cref{th:new} are given with the requirement that $\partial \psi$ and $\partial^{s+1} \psi$ are in $L^1(\R)$ respectively, for the embedding to hold it is sufficient, at least of integer $s$, that the (distributional) derivatives exist as measures.

The second point of \Cref{th:new} has been written in terms of Caputo derivatives. Whilst there exists many different fractional derivative [\cite{de_oliveira_review_2014}], each generalizing different properties of the classical derivative, this version allows us to take a Taylor extension with the proper form for the embedding to hold. Morever, the second point of \Cref{th:new} has been written in terms of a left and a right Caputo derivative. In the proof we use a Taylor expansion around zero. For classical derivatives the terms in the expansion are the same when we consider values above or below zero. For integer $s$ the Caputo derivatives $\phi^{s+1}_+$ and $\phi^{s+1}_-$ match with the classical derivative $\phi^{s+1}$ up to a sign. However, for non-integer $s$ we get that fractional derivatives, and in particular the Caputo derivatives, depend on the direction. Hence, we obtain in \Cref{th:new} point 2 a condition on $\phi^{s+1}_+$ and $\phi^{s+1}_-$ instead of just $\phi^{s+1}$.

Observe that the form of the integral in \eqref{eq:th_new} is similar to that of the integrals in \eqref{eq:push_forward_example}. This means we can interpret the embedding in point 1) of \Cref{th:new} as follows: If we have a shallow neural network with $\phi$ as activation function representing the function $f$, then we can replace each neuron in the hidden layer by a (possibly infinite) number of neurons with $\psi$ as activation function. $\gamma$ describes how the weights and biases of the neurons in the network should be adjusted. After the replacement, the network will still represent $f$. This interpretation is visualized in \Cref{fig:replacement}.

\begin{figure}
    \centering
    \includegraphics[width=0.8\textwidth]{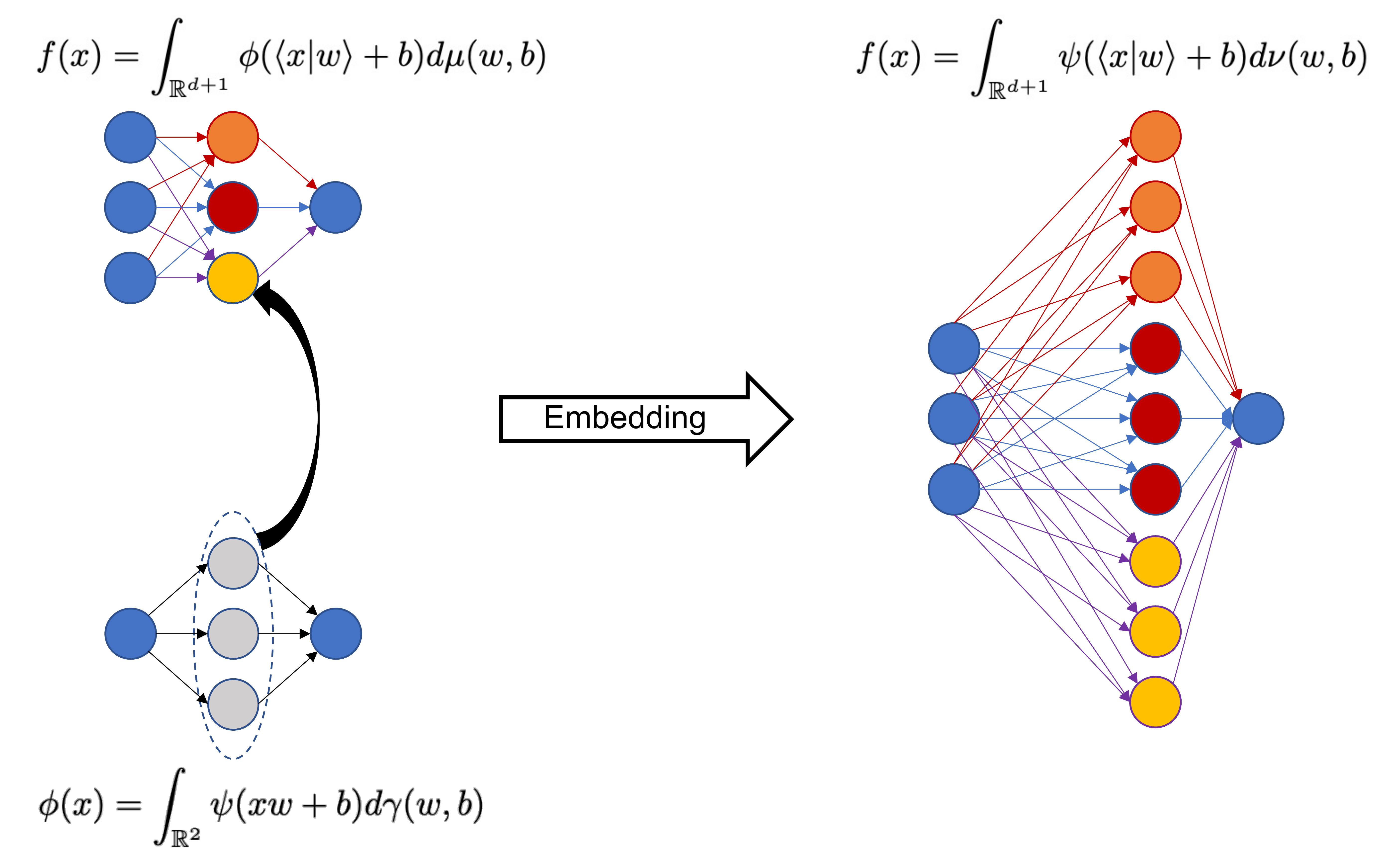}
    \caption{Each circle represents a neuron, and arrows represent connections between neurons. On the left, a network with $\phi$ as activation function representing $f$ is shown. The activation function $\phi$ can be represented using a shallow neural network with 3 neurons in the hidden layer and activation function $\psi$. On the right, a network with $\psi$ as activation function representing $f$ is shown. The network representing $\phi$ is used to construct the network on the right from that on the top left. Colors have been added to track which neuron on the right corresponds to which on the top left.}
    \label{fig:replacement}
\end{figure}

\subsection{Notation}
Denote with $\N$ the natural numbers without zero. Derivatives of $f$ are denoted with $\partial f$. When $f$ is multivariate, we use multi-indices to denote the partial derivatives. When working with fractional derivatives, we consider the left and right Caputo given by
\begin{subequations}
\begin{align}
    \partial^\alpha\phi_+(x) &= \frac{1}{\Gamma(n - \alpha)}\int_y^x(x-t)^{n-\alpha -1}\partial^n f(t)dt \qquad x \geq y\\
    \partial^\alpha\phi_-(x) &= \frac{(-1)^n}{\Gamma(n - \alpha)}\int_x^y(t-x)^{n-\alpha -1}\partial^n f(t)dt \qquad x \leq y
\end{align}    
\end{subequations}
for $n-1\leq \alpha<n$, respectively, with the boundary points $y$ made clear from context. Radon measures are regular signed Borel measures with bounded total variation. The space of Radon measures $\mathcal{M}(\Omega)$ on a locally compact Hausdorff space $\Omega$ is the continuous dual of the continuous functions vanishing at infinity, $C_0(\Omega)^\ast=\mathcal{M}(\Omega)$. The total variation measure of a measure $\mu\in \mathcal{M}(\Omega)$ is denoted by $\abs{\mu}$. The Dirac measure is given by
\begin{equation}
    \delta_{w}(A) = \begin{cases} 1 & w\in A \\ 0 & w\not\in A \end{cases}
\end{equation}
for Borel sets $A\subseteq \Omega$. For a map $\Theta$ defined by
\begin{equation}
    \Theta: X\to Y, \quad x\mapsto \Theta(x),
\end{equation}
we call the measure $\nu:=\Theta_{\#}\mu$ the push-forward of $\mu$ along the map $\Theta$ such that 
\begin{equation}
    \int_Y f(y) d\nu(y) = \int_X f(\Theta(x))d\mu(x),
\end{equation}
for all $\nu$-measurable functions $f$. A normed vector space $A$ embeds into another normed vector space $B$ if and only if $A\subseteq B$ and $\norm{f}_B \lesssim \norm{f}_A$ for all $f\in A$, where $\lesssim$ means that the inequality holds up to a constant $C>0$ independent of $f$. If $f\in L^1(\R^d)$, then
\begin{equation}
    \hat{f}(\xi) = \frac{1}{(2\pi)^d}\int_{\R^d}e^{-i\braket{x}{\xi}}f(x)d\xi,\quad \xi\in\R^d
\end{equation}
denotes the Fourier transform of $f$. As is common, we also use $\hat{f}$ if the Fourier transform exists in a generalized way.

\section{Embeddings between Barron spaces involving \texorpdfstring{$\repu$}{RePU}}
\label{sec:embeddings}
In this section, we start by defining the Barron spaces in \Cref{sec:barron_spaces}. We proceed by showing that the Barron spaces with $\repu$ as the activation function have a hierarchical structure. Next to that, we show that Barron spaces with an activation function for which the weak derivative is in $L^1(\R)$ embed into the Barron space with $\relu$ as activation function. After assuming $\Omega$ is bounded, we extend this to: Barron spaces with an activation function $\sigma\in C^{\lceil s \rceil}(\R)$ for which $\partial^{s+1}\sigma\in L^1(\R)$ embed in a Barron spaces with a $\repu_s$. The former we do in \Cref{sec:hierarchy} and the latter two in \Cref{sec:abs_cont_functions}.

\subsection{Barron spaces}\label{sec:barron_spaces}
The definition of the Barron spaces that we will be using is an adaption of Definition A.2 of [\cite{e_kolmogorov_2021}]. We use signed Radon measures instead of probability measures and have defined a natural norm for when the activation function is a $\repu$.

Fix $d\in\N$. Let $\X=[-1,1]^d$ and $\Omega\subseteq \R^{d+1}$. When we write $(w,b)\in\Omega$, we mean that $w\in\R^d$ and $b\in\R$. We use the $\ell^1$ norm for $(w,b)\in\Omega$ and the $\ell^\infty$ norm for $x\in\X$ so that $\norm{(w,b)}_{\ell^1}=\norm{w}_{\ell^1}+\abs{b}$ and $\abs{\braket{x}{w}}\leq \norm{w}_{\ell^1}$. We call $\sigma: \R\to\R$ an activation function when it is Lipschitz continuous or $\sigma(x)=\repu_s(x):=x^s\step(0,x)=\max(0,x)^s$ for some $s\geq 0$, where the last equality holds for all $s>0$ only. Consider functions $f:\X\to\R$ given by \begin{equation}
    f(x) = \int_\Omega \sigma(\braket{x}{w}+b)d\mu(w,b), \quad x\in \X.
\end{equation}
Since several distinct measures $\mu$ can describe the same function $f$, we group them into 
\begin{equation}
    \G_{\sigma,f|\Omega} = \Set[\bigg]{ \mu\in \M(\Omega) \given \forall x\in \X:\;f(x) = \int_\Omega \sigma(\braket{x}{w}+b) d\mu(w,b)}.
\end{equation}
The Barron norm is given by
\begin{equation}\label{eq:norm_lipschitz}
    \norm{f}_{\B_\sigma|\Omega} = \inf_{\mu\in \G_{\sigma,f|\Omega}}\int_\Omega(1+\norm{w}_{\ell^1}+\abs{b})d\abs{\mu}(w,b),
\end{equation}
unless $\sigma=\repu_s$ for some $s\geq 0$ in which case
\begin{equation}\label{eq:norm_repu}
    \norm{f}_{\B_\sigma|\Omega} = \inf_{\mu\in \G_{\sigma,f|\Omega}}\int_\Omega(\norm{w}_{\ell^1}+\abs{b})^sd\abs{\mu}(w,b).
\end{equation}
The Barron space for a given activation function $\sigma$ is given by
\begin{equation}
    \B_\sigma|\Omega = \Set[\bigg]{ f:\X \to \R \given \norm{f}_{\B_\sigma|\Omega} < \infty}.
\end{equation}
When $\Omega=\R^{d+1}$, we omit the $|\Omega$ in the expressions above.
\begin{remark}
With this choice of norm we get
\begin{equation}
    \abs{f(x)} \leq C(x,\sigma)\norm{f}_{\B_\sigma}
\end{equation}
for all $f\in\B_{\sigma}$ for both the Lipschitz activation functions and the $\repu$s. The integrands in \eqref{eq:norm_lipschitz} and \eqref{eq:norm_repu} are polynomial of order $p$, since $\sigma$ grows with order $p$. For the $\repu$ we have no $1+$ in the integrand, since we can rescale $\mu$ using the homogeneity of the $\repu$s so that $f$ stays the same but $\norm{\mu}_{\M(\Omega)}\to0$.
\end{remark}

\subsection{Hierarchy in the \texorpdfstring{$\repu$}{RePU}}\label{sec:hierarchy}
For continuous functions it is well-known that they have a hierarchical structure, i.e. $C^s\hookrightarrow C^t$ for all $s,t\in\N$ such that $t\leq s$. The following proposition shows that $\B_{\repu_s}$ has a similar hierarchy. It is a generalization of 1) from Lemma 7.1 of \cite{caragea_neural_2020}. 
\begin{proposition}\label{prop:repu_embedding}
For $0\leq s\leq t$ we have $\B_{\repu_t}\hookrightarrow\B_{\repu_s}$.
\end{proposition}
\begin{proof}
Let $c>0$. $\repu$ satisfies 
\begin{equation}
    \repu_t(x) = \frac{1}{\bF_1(t-s,s+1)}\int_0^cu^{t-s-1}\repu_s(x-u)du
\end{equation}
for all $\abs{x}\leq c$ and $0\leq s<t$, since
\begin{equation}
    \bF_1(t-s,s+1)x^{t} = \int_0^xy^{t-s-1}(x-y)^sdy
\end{equation}
for all $x>0$ with
\begin{equation}
    \bF_z(a,b) = \int_0^z t^{a-1}(1-t)^{b-1}dt.
\end{equation}
\\
Let $\mu\in \G_{\repu_{t},f}$ for $f\in \B_{\repu_{t}}$. Observe that
\begin{equation}
\begin{aligned}
    f(x) &= \int_{\R^{d+1}}\repu_{t}(\braket{x}{w}+b)d\mu(w,b) \\
    &= \int_{\R^{d+1}}\int_0^{\theta_{w,b}}\frac{u^{t-s-1}}{\bF_1(t-s,s+1)}\repu_s(\braket{x}{w}+b-u)dud\mu(w,b) \\
    &= \int_{\R^{d+1}}\int_0^{1}\frac{\theta_{w,b}^{t-s}v^{t-s-1}}{\bF_1(t-s,s+1)}\repu_s(\braket{x}{w}+b-\theta_{w,b}v)dvd\mu(w,b) && u=\theta_{w,b}v\\
    &= \int_{\R^{d+1}}\repu_t(\braket{x}{w}+b)d\nu(w,b),
\end{aligned}
\end{equation}
where $\theta_{w,b}:=\norm{w}_{\ell^1}+\abs{b}$, $\lambda$ is the Lebesgue measure on $[0,1]$ and the measure
\begin{equation}\label{eq:repu_hierarchy_measure}
    \nu(A) = \int_A\int_0^1\frac{\theta_{w,b}^{t-s}v^{t-s-1}}{\bF_1(t-s,s+1)}\Theta_{\#}(\mu\otimes\lambda)((w,b),v)
\end{equation}
for all Borel sets $A\subseteq\R^{d+1}$ is the push-forward of $\mu\otimes\lambda$ along the map 
\begin{equation}
    \Theta: \R^{d+1} \times [0,1] \to \R^{d+1}, \; ((w,b),v) \mapsto (w,b-\theta_{w,b}v).
\end{equation}
Hence, $\nu\in \G_{\repu_{s},f}$ for $f\in \B_{\repu_{s}}$. Furthermore,
\begin{equation}
\begin{aligned}
    \norm{f}_{\B_{\repu_s}} 
    &\leq \int_{\R^{d+1}}(\norm{w}+\abs{b})^sd\abs{\nu}(w,b) \\
    &= \int_{\R^{d+1}}\int_0^1\frac{\theta_{w,b}^{t-s}v^{t-s-1}}{\bF_1(t-s,s+1)}(\norm{w}+\abs{b-\theta_{w,b}v})^sdvd\abs{\mu}(w,b) \\
    &\leq \int_{\R^{d+1}}\int_0^1\frac{\theta_{w,b}^{t-s}v^{t-s-1}}{\bF_1(t-s,s+1)}(\norm{w}+\abs{b}+\theta_{w,b}v)^sdvd\abs{\mu}(w,b) \\
    &= \int_{\R^{d+1}}\int_0^1\frac{(1+v)^sv^{t-s-1}}{\bF_1(t-s,s+1)}(\norm{w}+\abs{b})^tdvd\abs{\mu}(w,b) \\
    &= (-1)^{s-t}\frac{\bF_{-1}(t-s,s+1)}{\bF_1(t-s,s+1)}\int_{\R^{d+1}}(\norm{w}+\abs{b})^td\abs{\mu}(w,b).
\end{aligned}
\end{equation}
Taking the infimum over $\mu\in \G_{\repu_{t},f}$ gives
\begin{equation}
    \norm{f}_{\B_{\repu_s}} \leq (-1)^{s-t}\frac{\bF_{-1}(t-s,s+1)}{\bF_1(t-s,s+1)}\norm{f}_{\B_{\repu_t}}.
\end{equation}
\end{proof}
Note that the embedding bound can be expressed equivalently as
\begin{equation}
    \norm{f}_{\B_{\repu_s}} \leq (-1)^{-\epsilon}\frac{\bF_{-1}(\epsilon,s+1)}{\bF_1(\epsilon,s+1)}\norm{f}_{\B_{\repu_{s+\epsilon}}}
\end{equation}
for all $\epsilon>0$ and all $s\geq 0$, with in particular
\begin{equation}
    \norm{f}_{\B_{\repu_s}} \leq (-1)^{-1}\frac{\bF_{-1}(1,s+1)}{\bF_1(1,s+1)}\norm{f}_{\B_{\repu_{s+1}}} = (2^{s+1}-1)\norm{f}_{\B_{\repu_{s+1}}}
\end{equation}
for all $s\geq 0$.

\subsection{Absolutely continuous activation functions}\label{sec:abs_cont_functions}
To show that the Barron spaces with an activation function $\sigma\in C^{\lceil s \rceil}(\R)$ for which $\partial^{s+1}\sigma\in L^1(\R)$ embed into a Barron space with a $\repu$ as activation function, we first prove some technical lemmas concerning $\repu_s$, which are important for the proofs of the embeddings. These lemmas will be reused in \Cref{sec:taylor}. 

Recall that the rectified power unit is given by
\begin{equation}
    \repu_s(x) = x^s\step(x)=\max(0,x)^s
\end{equation}
for $s\geq 0$ with the last equality holding for $s>0$ only. The embeddings proven in this section rely on the tie between $\repu_s$ and the Taylor remainder theorem to construct the push-forward map. The integral form of the Taylor remainder theorem states that a function $\phi\in C^{m}(\R)$ of which $\partial^{m}\phi$ is absolutely continuous on the closed interval $[a,b]$, can be written as
\begin{equation}\label{eq:taylor_remainder}
    \phi(x) = \underbrace{\sum_{k=0}^{m}\frac{\partial^k\phi(y)}{k!}(x-y)^k}_{\text{sum}} + \underbrace{\int_y^x\frac{\partial^{m+1}\phi(t)}{m!}(x-t)^{m}dt}_{\text{remainder}}
\end{equation}
for all $x,y\in[a,b]$. This well-known theorem follows straightforwardly from several applications of integration by parts. There are several different fractional Taylor expansions which differ in the type of fractional derivative chosen and whether the derivatives inside the sum are fractional. For the embedding with $s$ not integer, we choose the fractional Taylor expansion with Caputo derivatives and integer derivatives inside the sum [\cite{trujillo_riemannliouville_1999,odibat_generalized_2007}]. With this choice, it states that a function $\phi\in C^{\lceil s\rceil}(\R)$ of which $\partial^{s+1}\phi$ is locally integrable on the closed interval $[a,b]$, can be written as
\begin{equation}\label{eq:taylor_remainder_fractional}
    \phi(x) = \underbrace{\sum_{k=0}^{\lceil s\rceil}\frac{\partial^k\phi(y)}{k!}(x-y)^k}_{\text{sum}} + \underbrace{\begin{cases}
        \int_y^x\frac{\partial^{s+1}\phi_+(t)}{\Gamma(s+1)}(x-t)^{s}dt & x > y \\
        \int_x^y\frac{\partial^{s+1}\phi_-(t)}{\Gamma(s+1)}(t-x)^{s}dt & x < y \\
    \end{cases}}_{\text{remainder}}
\end{equation}
for all $x,y\in[a,b]$, where $\partial^{s+1}\phi_+$ and $\partial^{s+1}\phi_-$ are the left and right Caputo derivative given by
\begin{subequations}
\begin{align}
    \partial^\alpha\phi_+(x) &= \frac{1}{\Gamma(n - \alpha)}\int_y^x(x-t)^{n-\alpha -1}\partial^n f(t)dt \qquad x \geq y\\
    \partial^\alpha\phi_-(x) &= \frac{(-1)^n}{\Gamma(n - \alpha)}\int_x^y(t-x)^{n-\alpha -1}\partial^n f(t)dt \qquad x \leq y
\end{align}    
\end{subequations}
for $n-1\leq \alpha<n$ [\cite{de_oliveira_review_2014}]. For fixed $y$, both the series and remainder parts can be written in the form of \eqref{eq:integral_formulation} using a suitably chosen measure. To prove this, we use the following two lemmas. The first lemma deals with the sum part, whereas the second lemma deals with the remainder part. For the proofs of both lemma's, we refer the reader to the appendix.

\begin{lemma}\label{lemma:series_repu_s}
Let $p:\R^{d}\to \R$ be a polynomial of degree less or equal to $m\in \N$. Then there exists a measure $\nu\in\M(\R^{d+1})$ such that
\begin{equation}\label{eq:20}
    p(x) = \int_{\R^{d+1}} \repu_{m}(\braket{x}{w}+b)d\nu(w,b), \qquad x\in \R^{d}.
\end{equation}
\end{lemma}

\begin{lemma}\label{lemma:integral_bound_fix_lemma}
Let $s\geq 0$ and $c>0$. When $f\in L^1([-c,c])$, we have
\begin{equation}\label{eq:integral_bound_fix_lemma}
    \int_0^z f(u)(z-u)^sdu = \int_0^c f(u)\repu_s(z-u)+(-1)^{s-1}f(-u)\repu_s(-z-u)du
\end{equation}
for all $z\in [-c,c]$.
\end{lemma}

From \Cref{lemma:series_repu_s} and \Cref{lemma:integral_bound_fix_lemma} it follows that both the series part and the remainder part of a function satisfying the Taylor remainder theorem can be written in terms of $\repu_s$. Hence, it may seem logical that an embedding of $\B_{\phi}$ into $\B_{\repu_s}$ exists if $\phi$ satisfies the requirements for the Taylor remainder theorem. Without additional assumptions, this does not hold for all $s\geq 0$. For $0\leq s <1$, the Barron norm is not defined for all considered activation functions. Hence, we don't discuss this case any further in this work. For $s>1$, this does not hold due to the different exponent in the norm for $\repu_s$ compared to the exponent in the Barron norm for $\phi$. We discuss this later in this section in more detail. For $s=1$, this suggested embedding exists without additional assumptions. This is shown in the following proposition.
\begin{proposition}\label{prop:taylor_embedding}
If $\phi\in C^1(\R)$ such that $\partial^2\psi\in L^1(\R)$,
then we have $\B_{\phi}\hookrightarrow \B_{\relu}$ with
\begin{equation}\label{eq:gamma_tilde_embedding_bound}
    \norm{f}_{\B_{\relu}} \leq \gamma(\phi)\norm{f}_{\B_{\phi}}
\end{equation}
for all $f\in \B_\psi$, where
\begin{equation}\label{eq:gamma_tilde}
    \gamma(\phi) := \inf_{y\in\R}\bigg( \abs{\phi(y)}+2\abs{\partial\phi(y)}+2(1+\abs{y})\int_\R\abs{\partial^2\phi(z)}dz\bigg).
\end{equation}
\end{proposition}
\begin{proof}
From the triangle inequality, it follows immediately that 
\begin{equation}
    \abs{\braket{x}{w}+b-y} \leq \norm{w}_{\ell^1}+\abs{b}+\abs{y} := \theta_{w,b,y}
\end{equation}
for all $(x,y,w,b)\in [-1,1]^d\cross \R\cross \R^d\cross \R$. From the Taylor remainder theorem, it follows that for given $(w,b)\in \R^{d+1}$ and $y\in\R$
\begin{equation}
    \phi(\braket{x}{w}+b) = \phi(y) + \partial^1\phi(y)(\braket{x}{w}+b) + \int_y^{\braket{x}{w}+b}\partial^2\phi(t)(\braket{x}{w}+b)-t)dt
\end{equation}
for all $x\in\X$. After the change of coordinate $v=t-y$, this becomes
\begin{equation}\label{eq:28}
    \phi(\braket{x}{w}+b) = \phi(y) + \partial^1\phi(y)(\braket{x}{w}+b) + \int_0^{\braket{x}{w}+b-y}\partial^2\phi(v+y)(\braket{x}{w}+b)-v-y)dv.
\end{equation}
From \Cref{lemma:integral_bound_fix_lemma}, it follows that \eqref{eq:28} is equivalent to 
\begin{equation}
\begin{aligned}
    \phi(\braket{x}{w}+b) = \phi(y) + \partial^1\phi(y)(\braket{x}{w}+b) + \int_0^{\theta_{w,b,y}}&\partial^2\phi(v+y)\relu(\braket{x}{w}+b-v-y)\\&+\partial^2\phi(-v+y)\relu(\braket{x}{-w}-b-v+y)dv
\end{aligned}
\end{equation}
for all $x\in \X$. After the change of coordinate $\theta_{w,b,y}u=v$ and using \eqref{eq:identity}, this becomes
\begin{equation}\label{eq:phi_relu}
\begin{aligned}
    \phi(\braket{x}{w}+b) &= \phi(y)\relu(1) + \partial^1\phi(y)\bigg(\relu(\braket{x}{w}+b)-\relu(\braket{x}{-w}-b)\bigg) \\&+ \int_0^{1}\theta_{w,b,y}\partial^2\phi(\theta_{w,b,y}u+y)\relu(\braket{x}{w}+b-\theta_{w,b,y}u-y)\\&+\int_0^{1}\theta_{w,b,y}\partial^2\phi(-\theta_{w,b,y}u+y)\relu(\braket{x}{-w}-b-\theta_{w,b,y}u+y)du.
\end{aligned}
\end{equation}
Let $\mu\in \G_{\phi,f}$ for $f\in \B_{\phi}$. Observe that using \eqref{eq:phi_relu} as a substitution we get
\begin{equation}
    f(x) = \int_{\R^{d+1}}\phi(\braket{x}{w}+b)d\mu(w,b) = \int_{\R^{d+1}}\relu(\braket{x}{w}+b)d\nu(w,b), 
\end{equation}
where the measure $\nu=\sum_{i=1}^5\nu_i$ is the sum of measures formed from the measures
\begin{equation}
\begin{aligned}
    \nu_1 &= \phi(y)\mu(\Omega)\delta_{(0,1)} \\ 
    \nu_2 &= \partial^1\phi(y)\mu \\ 
    \nu_3 &= -\partial^1\phi(y)\Theta_{\#}^0\mu \\ 
    \nu_4(A) &= \int_A\int_0^1 \theta_{w,b,y}\partial^2\phi(\theta_{w,b,y}u+y)d\Theta_{\#}^1(\mu\otimes\lambda)((w,b),u)\\ 
    \nu_5(A) &= \int_A\int_0^1 \theta_{w,b,y}\partial^2\phi(-\theta_{w,b,y}u+y)d\Theta_{\#}^2(\mu\otimes\lambda)((w,b),u)\\  
\end{aligned}    
\end{equation}
for the Borel sets $A\subseteq\Omega$ and using the push-forward maps 
\begin{equation}
\begin{aligned}
    \Theta^0 &: \R^{d+}\to \R^{d+1}, \; (w,b)\to (-w,-b) \\ 
    \Theta^1 &: \R^{d+1}\cross[0,1]\to \R^{d+1}, \; ((w,b),u)\to (w,b-\theta_{w,b,y}u-y) \\ 
    \Theta^2 &: \R^{d+1}\cross[0,1]\to \R^{d+1}, \; ((w,b),u)\to (-w,-b-\theta_{w,b,y}u+y) \\ 
\end{aligned}    
\end{equation}
where $\lambda$ is the Lebesgue measure on $[0,1]$. Hence, $\nu\in \G_{\relu,f}$. Furthermore, for each $\nu_i$ we have
\begin{align}
    \int_{\R^{d+1}}(\norm{w}_{\ell^1}+\abs{b})d\abs{\nu_1}(w,b) &\leq \abs{\phi(y)}\abs{\mu}(\Omega) \leq  \abs{\phi(y)}\int_{\R^{d+1}}(1+\norm{w}_{\ell^1}+\abs{b})d\abs{\mu}(w,b)\\
    \int_{\R^{d+1}}(\norm{w}_{\ell^1}+\abs{b})d\abs{\nu_2}(w,b) &\leq \abs{\partial\phi(y)}\abs{\mu}(\Omega) \leq  \abs{\partial\phi(y)}\int_{\R^{d+1}}(1+\norm{w}_{\ell^1}+\abs{b})d\abs{\mu}(w,b)\\
    \int_{\R^{d+1}}(\norm{w}_{\ell^1}+\abs{b})d\abs{\nu_3}(w,b) &\leq \abs{\partial\phi(y)}\abs{\Theta^0_{\#}\mu}(\Omega) \leq  \abs{\partial\phi(y)}\int_{\R^{d+1}}(1+\norm{w}_{\ell^1}+\abs{b})d\abs{\mu}(w,b)
\end{align}
and
\begin{equation}
\begin{aligned}
    \int_{\R^{d+1}}&(\norm{w}_{\ell^1}+\abs{b})d(\abs{\nu_4}+\abs{\nu_5})(w,b) \\&\leq \int_{\R^{d+2}}\int_{-1}^1\theta_{w,b,y}\abs{\partial^2\phi(\theta_{w,b,y}u+y)}(1+\norm{w}_{\ell^1}+\abs{b}+\theta_{w,b,y}\abs{u}+\abs{y})dud\abs{\mu}(w,b) \\
    &\leq\sup_{(w,b)\in\R^{d+1}}\int_{-1}^1\theta_{w,b,y}\abs{\partial^2\phi(\theta_{w,b,y}u+y)}(1+\abs{u})(1+\abs{y})du\int_{\R^{d+1}}(1+\norm{w}_{\ell^1}+\abs{b})d\abs{\mu}(w,b) \\
    &\leq 2(1+\abs{y})\sup_{(w,b)\in\R^{d+1}}\int_{-\theta_{w,b,y}+y}^{\theta_{w,b,y}+y}\abs{\partial^2\phi(z)}dz\int_{\R^{d+1}}(1+\norm{w}_{\ell^1}+\abs{b})d\abs{\mu}(w,b) \\
    &= 2(1+\abs{y})\int_\R\abs{\partial^2\phi(z)}dz\int_{\R^{d+1}}(1+\norm{w}_{\ell^1}+\abs{b})d\abs{\mu}(w,b)
\end{aligned}
\end{equation}
where we used the change of coordinates $z=\theta_{w,b,y}u+y$. This means that by the triangle inequality
\begin{equation}\label{eq:taylor_embedding_pre_inf_bound}
\begin{aligned}    \norm{f}_{\B_{\relu}} 
    &\leq \int_{\R^{d+1}}(\norm{w}_{\ell^1}+\abs{b})d\abs{\nu}(w,b)  \\
    &\leq \sum_{i=1}^5\int_{\R^{d+1}}(\norm{w}_{\ell^1}+\abs{b})d\abs{\nu_i}(w,b)  \\
    &= \bigg(\abs{\phi(y)} + 2\abs{\partial^1\phi(y)}+2(1+\abs{y})\int_\R\abs{\partial^2\phi(z)}dz\bigg)\int_{\R^{d+1}}(1+\norm{w}_{\ell^1}+\abs{b})d\abs{\mu}(w,b).
\end{aligned}
\end{equation}
Taking the infimum over $\pi\in \G_{\phi,f}$ and $y\in\R$ gives \eqref{eq:gamma_tilde_embedding_bound}.
\end{proof}

Observe that \eqref{eq:gamma_tilde} is very similar to the definition of $\gamma$ used in Theorem 1 of [\cite{li_complexity_2020}],
\begin{equation}
    \gamma(\phi) = \inf_{y\in\R}\bigg(\abs{\phi(y)}+(\abs{y}+2)\abs{\partial\phi}+\int_\R\abs{\partial^2\phi(t)}(1+\abs{t})dt\bigg).
\end{equation}
The change to our version of $\gamma$ means that is sufficient for the function $\partial^2\phi(x)$ to go like $(1+\abs{x})^{-(1+\epsilon)}$ instead of like $(1+\abs{x})^{-(2+\epsilon)}$ for some $\epsilon>0$ and for large values of $x$. Hence, \Cref{prop:taylor_embedding} is satisfied for more activation functions than Theorem 1 of [\cite{li_complexity_2020}].

\Cref{prop:taylor_embedding} does not cover piecewise smooth activation functions, of which there are many. It is easy to check that a slight alteration to \Cref{prop:taylor_embedding} allows us to also cover activation functions that are smooth everywhere except at the origin. 

The right-hand side of \eqref{eq:taylor_embedding_pre_inf_bound} can be bounded so that
\begin{equation}
    \norm{f}_{\B_{\relu}} \lesssim \bigg(\norm{\phi}_{C^1(\R)}+\norm{\partial^2\phi}_{L^1(\R)}\bigg)\int_\Omega(1+\norm{w}_{\ell^1}+\abs{b})d\abs{\mu}(w,b)
\end{equation}
When we repeat the steps of the proof of \Cref{prop:taylor_embedding} for some $s>1$, we get a bound of the form
\begin{equation}
    \norm{f}_{\B_{\repu_s}} \lesssim \bigg(\norm{\phi}_{C^{\lceil s \rceil}(\R)}+\norm{\partial^{s+1}\phi}_{L^1(\R)}\bigg)\int_\Omega(1+\norm{w}_{\ell^1}+\abs{b})^{\lceil s \rceil}d\abs{\mu}(w,b)
\end{equation}
for all $\mu\in\G_{\phi,f}$ for $f\in\B_{\phi}$. If we want an embedding, then we need to get rid of the exponent $s$ in the integral on the right-hand side or we need to introduce the exponent in the norm of $\B_\phi$. In the following proposition we show the former by showing an embedding holds under the assumption that $\Omega$ is bounded. In \Cref{sec:discussion} we briefly discuss different exponents in the Barron norm.

\begin{proposition}\label{prop:taylor_embedding_relu_s}
Let $s>1$ and $\Omega$ bounded. If $\phi\in C^{\lceil s\rceil}(\R)$ such that $\partial^{s+1}_+\phi\in L^1((0,\infty))$
and $\partial^{s+1}_-\phi\in L^1((-\infty,0))$, then $\B_\phi|_{\Omega}\hookrightarrow\B_{\repu_{s}}$ where $\B_\phi|_{\Omega}$ is $\B_\phi$ with the parameters restricted to $\Omega$.
\end{proposition}
\begin{proof}
We will show the proof for $s$ not being integer. The proof for integer $s$ is similar. The main difference is that we don't need the step to go from $\repu_{\lceil s \rceil}$ to $\repu_s$ in the series part in the fourth line of \eqref{eq:taylor_series_repu_embedding}.

Let $\mu\in\G_{\phi,f}$ for $f\in\B_{\phi}$, and recall that from the fractional version of the Taylor remainder theorem, \eqref{eq:taylor_remainder_fractional}, it follows that for given $(w,b)\in \Omega$

\begin{equation}
\begin{split}
    \phi(\braket{x}{w}+b) = \sum_{k=0}^{\lceil s\rceil}\frac{\partial^k\phi(y)}{k!}(\braket{x}{w}+b)^k + \begin{cases}
        \int_0^{\braket{x}{w}+b}\frac{\partial^{s+1}\phi_+(t)}{\Gamma(s+1)}(\braket{x}{w}+b-t)^{s}dt & \braket{x}{w}+b > 0 \\
        \int_{\braket{x}{w}+b}^0\frac{\partial^{s+1}\phi_-(t)}{\Gamma(s+1)}(t-\braket{x}{w}+b)^{s}dt & \braket{x}{w}+b < 0 \\
    \end{cases}
\end{split}
\end{equation}
for all $x\in\X$.

For the series part, we can use the ideas from the proof of \Cref{prop:repu_embedding} and \Cref{lemma:series_repu_s} to conclude that there exists a measure $\nu_{sum}\in\M(\R^2)$ such that 
\begin{equation}
\begin{split}
    \int_\Omega &\sum_{k=0}^{\lceil s\rceil}\frac{\partial^k\phi(0)}{k!}(\braket{x}{w}+b)^kd\mu(w,b) \\
    &= \int_{\Omega}\int_{\R^2} \repu_{\lceil s \rceil}(\upomega (\braket{x}{w}+b)+\beta)d\nu_{{sum}}(\upomega,\beta)d\mu(w,b) \\
    &= \int_{\Omega}\int_{\R^2} \repu_{\lceil s \rceil}(\braket{x}{\upomega w}+\upomega b+\beta)d\nu_{{sum}}(\upomega,\beta)d\mu(w,b) \\
    &= \int_{\Omega}\int_{\R^2} \int_0^{\theta_{w,b,\upomega,\beta}}\frac{u^{\lceil s \rceil-s-1}}{\bF_1(\lceil s \rceil-s,s+1)}\repu_{s}(\braket{x}{\upomega w}+\upomega b+\beta-u)dud\nu_{{sum}}(\upomega,\beta)d\mu(w,b) \\
    &= \int_{\Omega}\int_{\R^2} \int_0^1\frac{\theta_{w,b,\upomega,\beta}^{\lceil s \rceil-s}v^{\lceil s \rceil-s-1}}{\bF_1(\lceil s \rceil-s,s+1)}\repu_{s}(\braket{x}{\upomega w}+\upomega b+\beta-\theta_{w,b,\upomega,\beta}v)dvd\nu_{{sum}}(\upomega,\beta)d\mu(w,b) \\
    &= \int_{\R^{d+1}}\repu_s(\braket{x}{w}+b)d\gamma_{sum,1}(w,b)
\end{split}
\end{equation}
where $\theta_{w,b\omega,\beta}:= \norm{\upomega w}+\abs{\upomega b}+\abs{\beta}$, $\lambda$ is the Lebesque measure on $[0,1]$ and the measure
\begin{equation}
    \gamma_{sum,1}(A) = \int_A\int_{\R^2}\int_0^1\frac{\theta_{w,b,\upomega,\beta}^{\lceil s \rceil-s}v^{\lceil s \rceil-s-1}}{\bF_1(\lceil s \rceil-s,s+1)}d\Theta_{\#}(\lambda\otimes\nu_{sum}\otimes\mu)(v,(\upomega,\beta),(w,b))
\end{equation}
for all Borel sets $A\subseteq \R^{d+1}$ uses the push-forward map
\begin{equation}
    \Theta: [0,1]\cross \R^2 \cross \Omega \to \R^{d+1},\; (v,(\upomega,\beta),(w,b)) \mapsto (\upomega w,\upomega b+\beta-\theta_{w,b,\upomega,\beta}v).
\end{equation}
This measure $\gamma_{sum,1}$ satisfies
\begin{equation}\label{eq:taylor_series_repu_embedding}
\begin{split}
    \int_{\R^{d+1}} &(\norm{w}_{\ell^1}+\abs{b})^sd\abs{\gamma_{sum,1}}(w,b) \\
    &\leq \int_\Omega\int_{\R^2}\int_0^1\frac{\theta_{w,b,\upomega,\beta}^{\lceil s \rceil-s}v^{\lceil s \rceil-s-1}}{\bF(\lceil s \rceil-s,s+1)}(\norm{\upomega w}_{\ell^1}+\abs{\upomega b+\beta-\theta_{w,b,\upomega,\beta}v})^sdvd\abs{\nu_{sum}}(\upomega,\beta)d\abs{\mu}(w,b) \\
    &\leq \int_\Omega\int_{\R^2}\int_0^1\frac{\theta_{w,b,\upomega,\beta}^{\lceil s \rceil-s}v^{\lceil s \rceil-s-1}}{\bF_1(\lceil s \rceil-s,s+1)}(\norm{\upomega w}_{\ell^1}+\abs{\upomega b}+\abs{\beta}+\theta_{w,b,\upomega,\beta}v)^sdvd\abs{\nu_{sum}}(\upomega,\beta)d\abs{\mu}(w,b) \\
    &= \int_\Omega\int_{\R^2}\int_0^1\frac{(1+v)^sv^{\lceil s \rceil-s-1}}{\bF_1(\lceil s \rceil-s,s+1)}(\norm{\upomega w}_{\ell^1}+\abs{\upomega b}+\abs{\beta})^{\lceil s \rceil}dvd\abs{\nu_{sum}}(\upomega,\beta)d\abs{\mu}(w,b) \\
    &= (-1)^{s-\lceil s \rceil}\frac{\bF_{-1}(\lceil s \rceil-s,s+1)}{\bF_1(\lceil s \rceil-s,s+1)}\int_\Omega\int_{\R^2}(\norm{\upomega w}_{\ell^1}+\abs{\upomega b}+\abs{\beta})^{\lceil s \rceil}d\abs{\nu_{sum}}(\upomega,\beta)d\abs{\mu}(w,b) \\
    &\lesssim \int_{\R^2}(\abs{\upomega}+\abs{\beta})^{\lceil s \rceil}d\abs{\nu_{sum}}(\upomega,\beta)\int_\Omega(1+\norm{w}_{\ell^1}+\abs{b})^{\lceil s \rceil}d\abs{\mu}(w,b) \\
    &\lesssim \norm{\phi}_{C^{\lceil s \rceil}(\R)}(\upomega,\beta)\int_\Omega(1+\norm{w}_{\ell^1}+\abs{b})d\abs{\mu}(w,b).
\end{split}
\end{equation}

For the remainder part, observe that 
\begin{equation}
    \abs{\braket{x}{w}+b} \leq \norm{w} + \abs{b} := \theta_{w,b}
\end{equation}
for all $x\in \X$ and $(w,b)\in\Omega$. The first case of the remainder satisfies
\begin{equation}\label{eq:remainder_part1}
\begin{split}
    \int_0^{\braket{x}{w}+b}\frac{\partial^{s+1}\phi_+(t)}{\Gamma(s+1)}(\braket{x}{w}+b-t)^{s}dt &= \int_0^{\braket{x}{w}+b} \frac{\partial^{s+1}_+\phi(t)}{\Gamma(s+1)}\repu_s(\braket{x}{w}+b-t)dt \\
    &= \int_0^{\theta_{w,b}} \frac{\partial^{s+1}_+\phi(t)}{\Gamma(s+1)}\repu_s(\braket{x}{w}+b-t)dt \\
    &= \int_0^{1} \frac{\partial^{s+1}_+\phi(\theta_{w,b}t)}{\Gamma(s+1)}\repu_s(\braket{x}{w}+b-\theta_{w,b}u)\theta_{w,b}du \\
\end{split}
\end{equation}
for all $\braket{x}{w}+b>0$, where we used the change of coordinates $t=\theta_{w,b}u$. The second case of the remainder satisfies
\begin{equation}\label{eq:remainder_part2}
\begin{split}
    \int_{\braket{x}{w}+b}^0\frac{\partial^{s+1}\phi_-(t)}{\Gamma(s+1)}(t-(\braket{x}{w}+b))^{s}dt &= \int_{\braket{x}{w}+b}^0\frac{\partial^{s+1}\phi_-(t)}{\Gamma(s+1)}\repu_s(t-(\braket{x}{w}+b))dt \\
    &= \int_{-\theta_{w,b}}^0\frac{\partial^{s+1}\phi_-(t)}{\Gamma(s+1)}\repu_s(t-(\braket{x}{w}+b))dt \\
    &= \int_{-1}^0\frac{\partial^{s+1}\phi_-(\theta_{w,b}u)}{\Gamma(s+1)}\repu_s(\theta_{w,b}u-(\braket{x}{w}+b))\theta_{w,b}du \\
    &= \int_{-1}^0\frac{\partial^{s+1}\phi_-(\theta_{w,b}u)}{\Gamma(s+1)}\repu_s(\braket{x}{-w}-b+\theta_{w,b}u)\theta_{w,b}du
\end{split}
\end{equation}
for all $\braket{x}{w}+b<0$, where we again used the change of coordinates $t=\theta_{w,b}u$. The two parts \eqref{eq:remainder_part1} and \eqref{eq:remainder_part2} are zero on the domain of the other, i.e. on $\braket{x}{w}+b\leq 0$ and $\braket{x}{w}+b\geq 0$ respectively. Hence, their sum is equal the remainder part for all $x\in\X$ given fixed $(w,b)\in\Omega$. The measures 
\begin{subequations}
\begin{align}
    \gamma_{rem,1}(A) :&= \int_A\int_0^1 \frac{\partial^{s+1}_+\phi(\theta_{w,b}u)}{\Gamma(s+1)}\theta_{w,b}d\Theta_{\#}^1(\mu\otimes\lambda)((w,b),u) \\
    \gamma_{rem,2}(A) :&= \int_A\int_{-1}^0\frac{\partial^{s+1}\phi_-(\theta_{w,b}u)}{\Gamma(s+1)}\theta_{w,b}d\Theta_{\#}^2(\mu\otimes\lambda)((w,b),u)
\end{align}    
\end{subequations}
for Borel sets $A\subseteq \Omega$ and using the push-forward maps 
\begin{equation}
\begin{aligned}
    \Theta^1 &: \Omega\cross[0,1]\to \R^{d+1}, \; ((w,b),u)\mapsto (w,b-\theta_{w,b}u) \\ 
    \Theta^2 &: \Omega\cross[0,1]\to \R^{d+1}, \; ((w,b),u)\mapsto (-w,-b-\theta_{w,b}u),
\end{aligned}    
\end{equation}
where $\lambda$ is the Lebesgue measure on $[0,1]$, satisfy
\begin{equation}
\begin{aligned}
    \int_{\R^{d+1}}(\norm{w}_{\ell^1}+\abs{b})^sd\abs{\gamma_{rem,1}}(w,b) &\lesssim \int_{\R^{d+1}}\int_0^1\theta_{w,b}\abs{\partial^{s+1}_+\phi(\theta_{w,b}u)}\bigg(\norm{w}_{\ell^1}+\abs{b-\theta_{w,b}u}\bigg)^sdud\abs{\mu}(w,b) \\
    &\leq \int_{\R^{d+1}}\int_0^1\theta_{w,b}\abs{\partial^{s+1}_+\phi(\theta_{w,b}u)}\bigg(\norm{w}_{\ell^1}+\abs{b}+\theta_{w,b}u\bigg)^sdud\abs{\mu}(w,b) \\
    &\lesssim \int_{\R^{d+1}}\int_0^1\theta_{w,b}\abs{\partial^{s+1}_+\phi(\theta_{w,b}u)}\bigg(\norm{w}_{\ell^1}+\abs{b}\bigg)^sdud\abs{\mu}(w,b) \\
    &= \int_{\R^{d+1}}\int_0^{\theta_{w,b}}\abs{\partial^{s+1}_+\phi(t)}\bigg(\norm{w}_{\ell^1}+\abs{b}\bigg)^sdtd\abs{\mu}(w,b) & u=\theta_{w,b}t \\
    &\leq \norm{\partial^{s+1}_+\phi}_{L^1((0,\infty)}\int_{\R^{d+1}}(1+\norm{w}_{\ell^1}+\abs{b})^sdtd\abs{\mu}(w,b) \\
    &\lesssim \norm{\partial^{s+1}_+\phi}_{L^1((0,\infty)}\int_{\R^{d+1}}1+\norm{w}_{\ell^1}+\abs{b}dtd\abs{\mu}(w,b)
\end{aligned}
\end{equation}
and similary
\begin{equation}
    \int_{\R^{d+1}}(\norm{w}_{\ell^1}+\abs{b})^sd\abs{\gamma_{rem,1}}(w,b) \lesssim \norm{\partial^{s+1}_-\phi}_{L^1((-\infty,0)}\int_{\R^{d+1}}1+\norm{w}_{\ell^1}+\abs{b}dtd\abs{\mu}(w,b)
\end{equation}
By combining these measures with $\gamma_{sum}$ we get $\gamma_{sum}+\gamma_{rem,1}+\gamma_{rem,2}\in\G_{\repu_s,f}$ with bound
\begin{equation}
    \norm{f}_{\B_{\repu_s}} \lesssim (\norm{\phi}_{C^s(\R)}+\norm{\partial^{s+1}_+\phi}_{L^1((0,\infty))}+\norm{\partial^{s+1}_-\phi}_{L^1((-\infty,0))})\norm{f}_{\B_\phi}
\end{equation}
where we took the infimum over $\mu\in\G_{\phi,f}$.
\end{proof}

\section{Embeddings between Barron spaces for Lipschitz activation functions}\label{sec:generic}
In the previous sections, we dealt with the relations between two Barron spaces when one of the Barron spaces had $\repu_s$ as the activation function. In this section, we take a broader perspective and look at the relations between Barron spaces with Lipschitz activation functions. In particular, we provide embeddings when the activation functions can be written as a linear combination of scaled and shifted versions of another activation function and as a convolution with another activation function in \Cref{sec:linear_comb_and_convolution} and when one of the two activation functions is the derivative of the other in \Cref{sec:derivative}. 

\subsection{Activation functions related by linear combinations and convolutions}\label{sec:linear_comb_and_convolution}
Let $\phi$ be a Lipschitz activation function. Since the values for the weights and biases are not restricted a priori, we expect that a Barron space with activation function
\begin{equation}\label{eq:52}
    \psi(x) = c_0\phi(c_1x+c_2) 
\end{equation}
for some $c_0,c_1\neq 0$ and $c_2\in\R$ is similar to that with $\phi$. At the same time, we expect that a Barron space with
\begin{equation}\label{eq:53}
    \psi(x) = \phi(x) - \phi(x-c_3)
\end{equation}
for $c_3\neq0$ embeds in that with $\phi$. Both of these and more are covered by the following proposition.

\begin{proposition}\label{prop:linear_comb_embedding}
If $\psi$ and $\phi$ are Lipschitz activation functions such that
\begin{equation}
    \phi(x) = \int_{\R^2}\psi(\upomega x +\beta)d\gamma(\upomega,\beta)
\end{equation}
for some measure $\gamma\in \M(\R^2)$ satisfying
\begin{equation}
    \int_{\R^2}(1+\abs{\upomega}+\abs{\beta})d\abs{\gamma}(\upomega ,\beta)<\infty,
\end{equation}
then $\B_\phi\hookrightarrow\B_\psi$.
\end{proposition}
\begin{proof}
Let $\mu\in \G_{\phi,f}$ for some $f\in \B_\phi$. This means that
\begin{equation}
\begin{aligned}
    f(x) 
    &= \int_{\R^{d+1}} \phi(\braket{x}{w}+b)d\mu(w,b) \\
    &= \int_{\R^{d+1}} \int_{\R^2}\psi(\upomega(\braket{x}{w}+b)+\beta)d\gamma(\upomega,\beta)d\mu(w,b) \\
    &= \int_{\R^{d+1}\cross\R^2}\psi(\braket{x}{\upomega w}+\upomega b+\beta)d(\gamma\otimes\mu)((\upomega,\beta),(w,b)) \\
    &= \int_{\R^{d+1}} \psi(\braket{x}{w}+b)d\nu(w,b)
\end{aligned}
\end{equation}
where the measure $\nu:=\Theta_{\#}(\gamma\otimes\mu)$ is the push forward along the map
\begin{equation}
    \Theta: \R^{d+1}\cross\R^2\to\R^{d+1},\; (\upomega,\beta),(w,b)) \to (\upomega w, \upomega b+\beta).
\end{equation}
Hence, $\nu\in \G_{\psi,f}$. Furthermore,
\begin{equation}
\begin{aligned}
    \norm{f}_{\B_\psi} 
    &\leq \int_{\R^{d+1}} (1+\norm{w}_{\ell^1}+\abs{b})d\abs{\nu}(w,b) \\
    &\leq \int_{\R^{d+1}}\int_{\R^2}(1+\norm{\upomega w}+\abs{\upomega b+\beta})d\abs{\gamma}(\upomega,\beta)d\abs{\mu}(w,b) \\
    &\leq \int_{\R^{d+1}}\int_{\R^2} (1+\abs{\upomega}\norm{w}_{\ell^1}+\abs{\upomega}\abs{ b}+\abs{\beta})d\abs{\gamma}(\upomega,\beta)d\abs{\mu}(w,b) \\
    &\leq \int_{\R^{d+1}}\int_{\R^2} (1+\abs{\upomega}+\abs{\beta})(1+\norm{w}_{\ell^1}+\abs{ b})d\abs{\gamma}(\upomega,\beta)d\abs{\mu}(w,b) \\
    &= \int_{\R^2}(1+\abs{\upomega}+\abs{\beta})d\abs{\gamma}(\upomega,\beta)\int_{\R^{d+1}}(1+\norm{w}_{\ell^1}+\abs{b})d\abs{\mu}(w,b).
\end{aligned}
\end{equation}
Taking the infimum over $\mu\in \G_{\psi,f}$ gives
\begin{equation}
    \norm{f}_{\B_\psi} \leq \int_{\R^2} (1+\abs{\upomega}+\abs{\beta})d\abs{\gamma}(\upomega,\beta)\norm{f}_{\B_\phi}. 
\end{equation}
\end{proof}
Informally, \Cref{prop:linear_comb_embedding} says that if $\phi$ is an element of the Barron space with $\psi$ as the activation function for $d=1$, then the embedding of $\B_\phi$ into $\B_\psi$ holds. This not only covers the aforementioned cases in \eqref{eq:52} and \eqref{eq:53} but also when $\phi$ is a convolution of $\psi$ with some kernel $\eta$ or when $\phi$ can be written as a series expansion in $\psi$.

 \begin{corollary}\label{cor:convolution}
If $\eta:\R\to\R$ satisfies
\begin{equation}
    \int_\R\abs{\eta(z)}(1+\abs{z})dz \leq C
\end{equation}
and $\phi$ is a Lipschitz activation function, then $\B_{\phi*\eta}\hookrightarrow\B_{\phi}$ with
\begin{equation}
    \norm{f}_{\B_{\phi}} \leq C\norm{f}_{\B_{\phi*\eta}}
\end{equation}
for all $f\in \B_{\phi*\eta}$.
\end{corollary}

\begin{corollary}
Let $\phi$ and $\psi$ be two Lipschitz activation functions linked by
\begin{equation}
    \phi(x) = \sum_{k=1}^\infty g(k)\psi(h(k)x)
\end{equation}
with
\begin{equation}
    \sum_{k=1}^\infty \abs{g(k)}(1+\abs{h(k)}) \leq C,
\end{equation}
then $\B_\phi\hookrightarrow\B_\psi$ with
\begin{equation}
    \norm{f}_{\B_{\psi}} \leq C\norm{f}_{\B_{\phi}}
\end{equation}
for all $f\in \B_{\phi}$.
\end{corollary}

Note that \Cref{cor:convolution} is particularly convenient if one knows the Fourier transforms of the relevant activation functions $\phi$ and $\psi$. In that case, it is sufficient to check whether the kernel $\eta$ defined using its Fourier transform
\begin{equation}
    \hat{\eta} :=\frac{\hat{\phi}}{\hat{\psi}}
\end{equation}
satisfies the growth condition. This allows one for example to show that
\begin{equation}
    \B_{\tanh} \hookrightarrow \B_{\arctan}.
\end{equation}
The corollary also shows that the cardinal B-splines $B_k$ are ordered. They are given by
\begin{equation}
\begin{aligned}
    B_{k} &= B_0 * B_{k-1}, \quad B_0(x) &= \begin{cases}
        1 & 0 \leq x \leq 1 \\
        0 & \text{otherwise}
    \end{cases}
\end{aligned}
\end{equation}
for all $k\in\N$. $B_0$ satisfies
\begin{equation}
    \int_\R\abs{B_0(z)}(1+\abs{z})dz = \frac{3}{2}.
\end{equation}
Hence, $\B_{B_k}\hookrightarrow\B_{B_{k-1}}$ for all $k\in\N$.

\subsection{Activation functions related by a derivative}\label{sec:derivative}
Something that \Cref{prop:linear_comb_embedding} does not cover, is when one activation function is the derivative of another. An example of this is
\begin{equation}
    SoftPlus(x) = \log(1+\e^x) 
\end{equation}
and 
\begin{equation}
    logi(x) = \frac{1}{1+\e^{-x}}
\end{equation}
related by
\begin{equation}
    \partial SoftPlus(x) = logi(x).
\end{equation}
In this case, only an inclusion has been found and not an embedding. 

\begin{proposition}\label{prop:diff_embedding}
If $\zeta$ is a continuously differentiable activation function with $Lip(\zeta)<\infty$, then $\B_\zeta\subseteq\B_{\partial\zeta}$.
\end{proposition}
\begin{proof}
Let $\mu\in \G_{\partial\zeta,f}$ for $f\in \B_{\partial\zeta}$. Consider the sequence of measures $\Set{v^h}_{h>0}$ given by
\begin{equation}
    \nu^h = \frac{1}{h}\bigg(\Theta^h_\#\mu-\mu\bigg)
\end{equation}
along the map
\begin{equation}
    \Theta^h: \R^{d+1}\to\R^{d+1}, \; (w,b) \mapsto (w,b+h).
\end{equation}
Observe that for
\begin{equation}
    f_h(x) = \int_{\R^{d+1}}\zeta(\braket{x}{w}+b)d\nu^h(w,b)
\end{equation}
we have
\begin{equation}
\begin{aligned}
    \norm{f_h}_{\B_{\zeta}} 
        &\leq \frac{1}{h}\int_{\R^{d+1}}(1+\norm{w}_{\ell^1}+\abs{b})d\abs{\Theta^h_\#\mu-\mu}(w,b) \\
        &\leq \frac{1}{h}\int_{\R^{d+1}}(1+\norm{w}_{\ell^1}+\abs{b})d\bigg(\abs{\Theta^h_\#\mu}+\abs{\mu}\bigg)(w,b) \\
        &= \frac{1}{h}\int_{\R^{d+1}}2(1+\norm{w}_{\ell^1})+\abs{b+h}+\abs{b}d\abs{\mu}(w,b) \\
        &\leq \frac{1}{h}\int_{\R^{d+1}} 2(1+\norm{w}_{\ell^1}+\abs{b})+\abs{h}d\abs{\mu}(w,b) \\
        &\leq \frac{2+h}{h}\int_{\R^{d+1}} 1+\norm{w}_{\ell^1}+\abs{b}d\abs{\mu}(w,b) \\ 
        & < \infty
\end{aligned}
\end{equation}
for all $h>0$. Hence, $f_h\in \B_{\zeta}$. The sequence $\Set{f_h}_{h>0}$ satisfies
\begin{equation}
    \begin{aligned}
        \lim_{h\to 0}f^h(x) 
        &= \lim_{h\to 0} \int_{\R^{d+1}}\zeta(\braket{x}{w}+b)d\nu^h(w,b) \\ 
        &= \lim_{h\to 0} \int_{\R^{d+1}} \frac{\zeta(\braket{x}{w}+b+h)-\zeta(\braket{x}{w}+b)}{h}d\mu(w,b) \\ 
        &= \int_{\R^{d+1}} \lim_{h\to 0} \frac{\zeta(\braket{x}{w}+b+h)-\zeta(\braket{x}{w}+b)}{h}d\mu(w,b) && \text{Dominated conv. th.}\\
        &= \int_{\R^{d+1}} \partial\zeta(\braket{x}{w}+b)d\mu(w,b) \\ 
        &= f(x), 
    \end{aligned}
\end{equation}
where we are allowed to use the dominated convergence theorem since
\begin{equation}
    \abs{\frac{\zeta(\braket{x}{w}+b+h)-\zeta(\braket{x}{w}+b)}{h}} \leq Lip(\zeta) < \infty.
\end{equation}
Since the Barron space $\B_{\zeta}$ is complete and $f_h\to f$, we have that $f\in \B_{\zeta}$.
\end{proof}

\section{Embeddings for spectral Barron spaces}\label{sec:taylor}
We have shown that embeddings between different Barron spaces can be proven by constructing suitable the push-forwards. This strategy can also be used to show embeddings between a Barron space and a non-Barron space. We will demonstrate this in this section by showing the embedding of the spectral Barron spaces into the Barron spaces with a $\repu_s$ as activation function. This embedding is a generalization of 3) from Lemma 7.1 of [\cite{caragea_neural_2020}].

We recall that the spectral Barron spaces are given by
\begin{equation}
\begin{aligned}
    \G^{\FS}_{m,f} &= \Set[\bigg]{ f\in L^1([-1,1]^d) \given \exists f_e\in L^1(\R^d):\; f_e|_{\X}=f } \\
    \norm{f}_{\B_{\FS,m}} &= \inf_{f_e\in \G^{\FS}_{m,f}}\int_{\R^d}(1+\norm{\xi}_{\ell^1})^{m}\abs{\hat{f_e}(\xi)}d\xi \\
    \B_{\FS,m} &= \Set[\bigg]{ f: L^1([-1,1]^d) \given \norm{f}_{\B_{\FS,m}} < \infty }
\end{aligned}
\end{equation}
for $m\in\N$, and have been called Spectral spaces and Auxiliary spaces as well. From Lemma 2.7 of [\cite{voigtlaender_lp_2022}] it follows that for each $s\in\N$ all the functions $f\in\B_{\FS,m+1}$ satisfy the conditions for the multivariate Taylor remainder theorem, i.e.
\begin{equation}\label{eq:f_remainder_multi}
    f(x) = \sum_{\abs{\alpha}\leq m}\frac{\partial^\alpha f(0)}{\alpha!}x^\alpha+\sum_{\abs{\alpha}=m+1}\frac{m+1}{\alpha!}x^\alpha\int_0^1(1-t)^{m}D^\alpha f(tx)dt,
\end{equation}
holds, where we have used the multi-index notation for $\alpha$. Similar to the univariate case, we can use to \Cref{lemma:series_repu_s} to construct a suitable push-forward map for the series part. However, unlike the univariate case, there is no analogue to \Cref{lemma:integral_bound_fix_lemma} to help us construct a push-forward map for the remainder part. Fortunately, we can construct one by using the spectral nature of $f\in\B_{\FS,m+1}$.

\begin{proposition}\label{prop:fourier_analytic_repu_embedding}
For all $m\in\N$ it holds that $\B_{\FS,m+1}\hookrightarrow\B_{\repu_{m}}$.
\end{proposition}
\begin{proof}
Let $f_e\in \G^{\FS}_{m,f}$ for $f\in\B_{\FS,m+1}$. Recall that \begin{equation}\label{eq:fourier_inversion_formula}
    f(x) = \int_{\R^d}\e^{i\braket{x}{\xi}}\hat{f_e}(\xi)d\xi
\end{equation}
for all $x\in \X$. The integral form of the Taylor remainder theorem for the exponential map $z\mapsto \e^{iz}$ around the origin up to order $m$ is given by
\begin{equation}\label{eq:exp_taylor}
    \e^{iz} = \sum_{k=0}^{m}\frac{i^k}{k!}z^k + \int_0^z\frac{i^{m+1}\e^{it}}{m!}(z-t)^{m}dt.
\end{equation}
Substituting \eqref{eq:exp_taylor} into the right-hand side of \eqref{eq:fourier_inversion_formula} gives
\begin{equation}\label{eq:fourier_inversion_formula_taylored}
    \int_{\R^d}\e^{i\braket{x}{\xi}}\hat{f_e}(\xi)d\xi = \int_{\R^d}\sum_{k=0}^{m}\frac{i^k}{k!}\braket{x}{\xi}^k\hat{f_e}(\xi)d\xi + \int_{\R^d}\int_0^{\braket{x}{\xi}}\frac{i^{m+1}\e^{it}}{m!}(\braket{x}{\xi}-t)^{m}dt\hat{f_e}(\xi)d\xi.
\end{equation}
For the series part, we observe that by the Fourier derivation identity 
\begin{equation}\label{eq:fourier_derivation_identity_applied}
    \int_{\R^d}\sum_{k=0}^{m}\frac{i^k}{k!}\braket{x}{\xi}^k\hat{f_e}(\xi)d\xi = \sum_{\abs{\alpha}\leq m}\frac{\partial^\alpha f(0)}{\alpha!}x^\alpha.
\end{equation}
For the remainder part, we observe that $\abs{\braket{x}{\xi}}\leq \norm{\xi}_{\ell^1}$, thus by \Cref{lemma:integral_bound_fix_lemma}
\begin{equation}
    \int_0^{\braket{x}{\xi}}\frac{i^{m+1}\e^{it}}{m!}(x-t)^{m}dt = \frac{i^{m+1}}{m!}\int_0^{\norm{\xi}_{\ell^1}}\bigg(\e^{it}\repu_{m}(\braket{x}{\xi}-t)+(-1)^{m-1}\e^{-it}\repu_{m}(-\braket{x}{\xi}-t)\bigg)dt.
\end{equation}
After doing the change of coordinates $u=\norm{\xi}_{\ell^1}t$ and substituting the resultant expression together with \eqref{eq:fourier_derivation_identity_applied} into the right-hand side of \eqref{eq:fourier_inversion_formula_taylored}, we get
\begin{equation}\label{eq:f_fourier_expansion_s_double_terms}
\begin{aligned}
    f(x) = \sum_{\abs{\alpha}\leq m}\frac{\partial^\alpha f(0)}{\alpha!}x^\alpha + \int_{\R^d}\int_{0}^1&\frac{i^{m+1}\norm{\xi}_{\ell^1}^{m+1}\hat{f_e}(\xi)}{m!}\bigg(\e^{i\norm{\xi}_{\ell^1}u}\repu_{m}\bigg(\braket{x}{\frac{\xi}{\norm{\xi}_{\ell^1}}}-u\bigg)\\&+\e^{-i\norm{\xi}_{\ell^1}u}\repu_{m}\bigg(\braket{x}{\frac{-\xi}{\norm{\xi}_{\ell^1}}}-u\bigg)\bigg)dud\xi.
\end{aligned}
\end{equation}
We can remove the second $\repu_{m}$ term by observing that
\begin{equation}\label{eq:f_fourier_expansion_s_substitute}
\begin{aligned}
    \int_{\R^d}\int_{0}^1&\frac{i^{m+1}\norm{\xi}_{\ell^1}^{m+1}\hat{f_e}(\xi)}{m!}\e^{-i\norm{\xi}_{\ell^1}u}\repu_{m}\bigg(\braket{x}{\frac{-\xi}{\norm{\xi}_{\ell^1}}}-u\bigg)dud\xi \\&= -\int_{\R^d}\int_{0}^1\frac{i^{m+1}\norm{\xi}_{\ell^1}^{m+1}\hat{f_e}(-\xi)}{m!}\e^{-i\norm{\xi}_{\ell^1}u}\repu_{m}\bigg(\braket{x}{\frac{\xi}{\norm{\xi}_{\ell^1}}}-u\bigg)dud\xi,
\end{aligned}
\end{equation}
where we used the coordinate map $\xi\mapsto -\xi$. Substituting \eqref{eq:f_fourier_expansion_s_substitute} into \eqref{eq:f_fourier_expansion_s_double_terms} gives
\begin{equation}\label{eq:f_fourier_expansion_s}
\begin{split}
    f(x) &= \sum_{\abs{\alpha}\leq m}\frac{\partial^\alpha f(0)}{\alpha!}x^\alpha \\\;&+ \int_{\R^d}\int_{0}^1\frac{i^{m+1}\norm{\xi}_{\ell^1}^{m+1}}{m!}\bigg(\hat{f_e}(\xi)\e^{i\norm{\xi}_{\ell^1}u}-\hat{f_e}(-\xi)\e^{-i\norm{\xi}_{\ell^1}u}\bigg)\repu_{m}\bigg(\braket{x}{\frac{\xi}{\norm{\xi}_{\ell^1}}}-u\bigg)dud\xi.
\end{split}
\end{equation}
From \Cref{lemma:series_repu_s} it follows that there exists a measure $\mu_{sum}\in \M(\Omega)$ such that
\begin{equation}
    \sum_{\abs{\alpha}\leq m}\frac{\partial^\alpha f(0)}{\alpha!}x^\alpha = \int_\Omega\repu_{m}(\braket{x}{w}+b)d\mu_{sum}(w,b).
\end{equation}
Simultaneously, we observe that
\begin{equation}
    f(x) = \int_{\Omega}\repu_{m}(\braket{x}{w}+b)d\mu(w,b),
\end{equation}
where the measure $\mu:=\mu_{sum}+\mu_{rem}$ is the sum of the measure $\mu_{sum}$ and the measure $\mu_{rem}$ given by
\begin{equation}
    d\mu_{rem}(\xi,u) = \int_0^1Re\bigg(\frac{i^{m+1}\norm{\xi}_{\ell^1}^{m+1}}{m!}\bigg(\hat{f_e}(\xi)\e^{i\norm{\xi}_{\ell^1}u}-\hat{f_e}(-\xi)\e^{-i\norm{\xi}_{\ell^1}u}\bigg)\bigg)d\Theta_{\#}(\lambda_{\R^{d}}\otimes\lambda_{[0,1]})(\xi,u)
\end{equation}
defined using the push-forward map
\begin{equation}
    \Theta: \R^d\cross[0,1] \to \S^d\cross[0,1], \; (\xi,u) \mapsto (\frac{\xi}{\norm{\xi}_{\ell^1}},u)
\end{equation}
with $\lambda_{\R^{d}}$ and $\lambda_{[0,1]}$ the Lebesgue measures on $\R^d$ and $[0,1]$ respectively. Hence, $\mu\in \G_{\repu_{m},f}$. Furthermore,
\begin{equation}
\begin{aligned}
    \norm{f}_{\B_{\repu_{m}}} 
    &\leq \int_{\Omega}(\norm{w}_{\ell^1}+\abs{b})^{m}d\abs{\mu}(w,b)  \\
    &\leq \int_{\Omega}(\norm{w}_{\ell^1}+\abs{b})^{m}d\abs{\mu_{sum}}(w,b)+\int_{\Omega}(\norm{w}_{\ell^1}+\abs{b})^{m}d\abs{\mu_{rem}}(w,b)  \\
    &\lesssim \norm{f_e}_{C^{m}_0(\R^d)}+\norm{\mu_{rem}}_{\M(\Omega)}  \\    &\lesssim \norm{\hat{f_e}}_{L^1(\R^d,(1+\norm{\cdot})^{m+1})},
\end{aligned}
\end{equation}
where we used Lemma 2.7 of [\cite{voigtlaender_lp_2022}] to bound $\norm{f_e}_{C^{m}_0(\R^d)}$. Taking the infimum over $f_e\in \G^{\FS}_{m,f}$ gives
\begin{equation}
    \norm{f}_{\B_{\repu_{m}}} \lesssim \norm{f}_{\B_{\FS,m+1}}.
\end{equation}
\end{proof}

\section{Discussion and conclusion}\label{sec:discussion}
In this paper, we have studied the effect of changing the activation function on the Barron spaces. This has been done by determining embeddings between two Barron spaces with different activation functions. 

We have shown that the Barron spaces with $\repu_s$ have a hierarchical structure, i.e. if $0\leq t \leq s$, then the Barron space with $\repu_s$ embeds into that with $\repu_t$. This structure is similar to well-known Sobolev spaces $H^s$ and the continuous function spaces $C^s$. In [\cite{e_observations_2022}], four PDEs with explicit formulas for their solutions are studied. These formulas can be derived using the Green's function associated with the PDE. They discuss several challenges when using Barron functions for the initial conditions and/or boundary conditions. When using Sobolev spaces, many of these challenges are overcome by assuming higher regularity. Some remaining challenges can potentially be solved by assuming higher $s$ for the Barron spaces with $\repu_s$. 

The embeddings, for which we assume that neither is a $\repu_s$, cover many of the changes that are made to existing activation function in order to find new ones to use. Examples of such changes are scaling and shifting (compare $logi$ with $\tanh$), taking a linear combination (compare leaky $\relu$ with $\relu$) and taking a derivative (compare $SoftPlus$ with $logi$). 

\begin{remark}
Throughout this work, we use $\X=[-1,1]^d$. This allows us to remove the dependence on $x$ when it appears in the bound of an integral. Relaxing this to, for instance $\R^d$, breaks many proofs. In particular, it breaks the hierachy of the RePU-Barron spaces. $\repu_s$-Barron functions grow at infinity at most like $\norm{x}^s$. This means that a function like $\repu_{s+1}$ is not in this space. We also used $\Omega=\R^{d+1}$. In contrast to the case with $\X$, this can be relaxed to smaller sets. However, the proof techniques used in this work suggest that, when one Barron space embeds into another, the latter will have a larger parameter space $\Omega$ than the former. 
\end{remark}

Although our results cover many activation functions, we have only provided affirmative statements, i.e. we provided statements that show a suitable push-forward map $\Theta$ exists and we provided no statements that show no such map can exist. To show that a particular embedding cannot exist, we can use the smoothness of the activation functions. If $\phi\in C^1(\R)$, then all functions
\begin{equation}
    \psi(x) = \int_{\R^{2}}\phi(xw+b)d\mu(w,b), \quad \int_{\R^{2}}\abs{w}+\abs{b}d\abs{\mu}(w,b)< \infty
\end{equation}
are $C^1(\R)$ too. This rules out embeddings of $\B_\psi$ into $\B_\phi$ when $\psi$ is at most Lipschitz continuous. Consider as an example the sawtooth wave function $\sawtooth_{A,p}$ with amplitude $A$ and period $p$ as activation function. This function has been used to show the relevance of depth in neural networks with $\relu$ as the activation function [\cite{telgarsky_representation_2015}]. It satisfies the identities
\begin{align}    
    \sawtooth_{A,p}(x) &= A\bigg( \frac{1}{2} - \frac{1}{\pi}\sum_{k=1}^\infty (-1)^k\frac{\sin(2\pi k p x)}{k}  \bigg) \label{eq:sawtooth_as_sin}\\
    \sin(x) &= \frac{1}{2}\int_0^{2\pi}\cos(y)\sawtooth_{1,2\pi}(x-y-\pi)dy
\end{align}
We have that $\B_{\sawtooth_{A,p}}$ is not included in $\B_{\sin}$ by the smoothness argument and an embedding of $\B_{\sin}$ into $\B_{\sawtooth_{A,p}}$ by \Cref{prop:linear_comb_embedding}.

Note that the $\sin$-Barron space does not the function $\sawtooth_{A,p}$, but it does contain the partial sums
\begin{equation}
    \sawtooth^N_{A,p}(x) = A\bigg( \frac{1}{2} - \frac{1}{\pi}\sum_{k=1}^N (-1)^k\frac{\sin(2\pi k p x)}{k}\bigg).
\end{equation}
for all finite $N$. This suggest that we can approximate $\sawtooth_{A,p}$-Barron functions using $\sin$-Barron functions. This idea is explored using the concept \textit{Barron approximation space}. We refer the reader to \cite{caragea_neural_2020} for more information about these spaces.

Our results are also limited to $\repu_s$ or Lipschitz activation functions. This restriction makes sure that, given an activation function $\sigma$, a function $f\in\B_{\sigma}$ of the form \eqref{eq:integral_formulation} is well-defined for all $\mu\in\G_{\sigma,f}$. An activation function like
\begin{equation}
    \sigma(x) = \abs{x}^2
\end{equation}
is not covered by this [\cite{sarao_mannelli_optimization_2020}]. This activation function is asymptotically quadratic and is thus not Lipschitz. To cover activation functions like this the Barron spaces can be adapted by redefining the Barron norm for continuous non-homogeneous functions as 
\begin{equation}
    \norm{f}_{\B_\sigma} = \inf_{\mu\in\G_{\sigma,f}}\int_\Omega(1+\norm{w}_{\ell^1}+\abs{b})^{p}d\abs{\mu}(w,b)
\end{equation}
with
\begin{equation}
    p = \argmin\Set[\bigg]{ q\in\N \given \forall x\in\X,(w,b)\in\Omega: \; \frac{\abs{\sigma(\braket{x}{w}+b)}}{(1+\norm{w}_{\ell^1}+\abs{b})^{q}}<\infty }.
\end{equation}
When $\sigma$ is Lipschitz continuous, $p=1$. Hence, this recovers the Barron norm in that case. Note that results like \Cref{prop:taylor_embedding} are still preserved, since for all $p\in\N$ we have
\begin{equation}
    \inf_{\mu\in\G_{\sigma,f}}\int_\Omega(1+\norm{w}_{\ell^1}+\abs{b})^{p}d\abs{\mu}(w,b) \geq \inf_{\mu\in\G_{\sigma,f}}\int_\Omega(1+\norm{w}_{\ell^1}+\abs{b})d\abs{\mu}(w,b).
\end{equation}

\section*{Acknowledgements}
TJH and CB acknowledge support by Sectorplan Bèta (the Netherlands) under the focus area \enquote{Mathematics of Computational Science}. CB acknowledges support by the European Union's Horizon 2020 research and innovation programme under the Marie Skłodowska-Curie grant agreement No 777826 (NoMADS). TJH thanks José Iglesias for the fruitful discussions, and we thank the reviewers for their valuable feedback.

\printbibliography

\appendix
\section{Taylor to RePU proofs}\label{sec:appendix_A}
This lemma contains the proofs for \Cref{lemma:series_repu_s} and \Cref{lemma:integral_bound_fix_lemma}. These are repeated here as \Cref{lemma:series_repu_s_appendix} and \Cref{lemma:integral_bound_fix_lemma_appendix} respectively.

To prove \Cref{lemma:series_repu_s_appendix} we make use of \Cref{lemma:change_of_basis}. This lemma is similar to Theorem 2 in [\cite{chen_power_2022}]. However, the proof for the full statement in [\cite{chen_power_2022}] is contained in a currently unpublished paper. Hence, for completeness, we have provided this proof.

\begin{lemma}\label{lemma:change_of_basis}
Let $m,d\in\N$. There exist $p:=\binom{m+d}{d}$ pairs $(w_i,b_i)\in\R^{d+1}$ with $i\in \Set{1,\hdots,p}$ such that the set of polynomials $\Theta_{m}:=\Set[\bigg]{(\braket{x}{w_1}+b_1)^{m},(\braket{x}{w_2}+b_2)^{m},\hdots,(\braket{x}{w_p}+b_p)^{m}}$ forms a basis for the space of polynomials in $d$ variables with degree at most $m$.
\end{lemma}
\begin{proof}
There are $p$ multi-indices with total degree at most $s$. Let the sequence $\Set{\alpha_n}_{n=1}^p$ be the set with these multi-indices in inverse lexicographical order. The statement holds if we can choose the pairs $(w_i,b_i)\in\R^{d+1}$ such that there exists an invertible matrix $W$ which satisfies
\begin{equation}\label{eq:change_of_basis}
    WX = P_{m}
\end{equation}
where $X=\mqty( x^{\alpha_1}, \hdots, x^{\alpha_p})^\intercal$ and $P_s=\bigg((\braket{x}{w_1}+b_1)^{m},\hdots,(\braket{x}{w_p}+b_p)^{m}\bigg)^\intercal$. We will construct $W$ and use the theory of generalized Vandermonde matrices to show that it is invertible [\cite{gantmacher_theory_2009}].

Observe that for $(w,b)\in\Omega$ and $x\in\X$ we have by simple combinatorics that
\begin{equation}
\begin{aligned}
    (\braket{x}{w}+b)^{m} 
    &= (\braket{(x,1)}{(w,b)})^{m} \\
    &= \sum_{\abs{\beta}=m}\binom{m}{\beta}(w,b)^\beta(x,1)^\beta \\
    &= \sum_{\abs{\gamma}\leq m}\binom{m}{\abs{\gamma}}\binom{\abs{\gamma}}{\gamma}w^\gamma b^{m-\abs{\gamma}} x^\gamma \\
    &= \sum_{n=1}^p\binom{m}{\abs{\alpha_n}}\binom{\abs{\alpha_n}}{\alpha_n}w^{\alpha_n} b^{m-\abs{\alpha_n}} x^{\alpha_n},
\end{aligned}    
\end{equation}
where $\beta$ and $\gamma$ are multi-indices of length $d+1$ and length $d$ respectively. This shows that a matrix $W$ with elements of the form $W_{ij}=\binom{m}{\abs{\alpha_j}}\binom{\abs{\alpha_j}}{\alpha_j}w_i^{\alpha_j} b_i^{m-\abs{\alpha_j}}$ satisfies \eqref{eq:change_of_basis}. What remains is choosing each $(w_i,b_i)$ such that $W$ is invertible.

If $\Tilde{W}$ is a matrix with elements $w_i^{\alpha_j} b_i^{m-\abs{\alpha_j}}$ and $D$ a diagonal matrix with entries $D_{jj}=\binom{m}{\abs{\alpha_j}}\binom{\abs{\alpha_j}}{\alpha_j}$, then
\begin{equation}
    det(W) = det(D)det(\Tilde{W}).
\end{equation}
Clearly, $det(D)>0$. If we take $b_i=1$, $1 < w_{1,1} < w_{2,1} < \hdots w_{p,1} < \infty$, and $w_{i,k}=w^{1+(k-1)\sqrt{\prime(k)}}_{i,1}$, where $\prime(k)$ is the $k^{\text{th}}$ prime number, for all $i\in\Set{1,\hdots p}$, then each element of $\Tilde{W}$ is of the form
\begin{equation}
    \Tilde{W}_{ij} = w_{i,1}^{\abs{\alpha_j}+\sum_{k=1}^d(k-1)\sqrt{\prime(k)}\alpha_{j,k}}.
\end{equation}
The bases are fixed columnwise, but strictly increasing rowwise. The exponents are fixed rowwise, but distinct columnwise. Let $\Tilde{\Tilde{W}}$ be $\Tilde{W}$ with its columns reordered such that the exponents are in increasing order. By construction, $\Tilde{\Tilde{W}}$ is a generalized Vandermonde matrix. These matrices have a non-zero determinant [page 99 of \cite{gantmacher_theory_2009}]. Reordering the columns at most switches the sign of the determinant. Hence, $\Tilde{W}$ is invertible and thus $W$ is too.
\end{proof}

Using this lemma, we are now able to prove the following equality.

\begin{lemma}\label{lemma:series_repu_s_appendix}
Let $p:\R^{d}\to \R$ be a polynomial of degree less or equal to $m\in \N$. Then there exists a measure $\nu\in\M(\R^{d+1})$ such that
\begin{equation}\label{eq:20_appendix}
    p(x) = \int_{\R^{d+1}} \repu_{m}(\braket{x}{w}+b)d\nu(w,b), \qquad x\in \R^{d}.
\end{equation}
\end{lemma}
\begin{proof}
We can use \Cref{lemma:change_of_basis} to write the polynomial $p$ as a linear combination of the basis functions in $\Theta_s$, 
\begin{equation}
     p(x) = \sum_{i=1}^p \kappa_i(\braket{x}{w_i}+b_i)^{m}
\end{equation}
where $\kappa_i\in\R$. Combined with the identity
\begin{equation}\label{eq:identity}
    z^{m} = \repu_{m}(z)+(-1)^{m-1}\repu_s(-z)
\end{equation}
for all $z\in\R$, we can conclude that the measure $\nu=\nu_1+\nu_2$ defined using
\begin{equation}
\begin{aligned}
    \nu_1 &= \sum_{i=1}^p\kappa_i\delta_{(w_i,b_i)} \\
    \nu_2 &= \sum_{i=1}^p(-1)^{m-1}\kappa_i\delta_{(-w_i,-b_i)}
\end{aligned}
\end{equation}
satisfies \eqref{eq:20_appendix}.
\end{proof}

There are several things to note regarding (the proofs of) \Cref{lemma:change_of_basis} and \Cref{lemma:series_repu_s}. First, we chose $w_i$ with $\norm{w_i}_{\ell^1} \leq d\abs{w_{i,1}}^{(d-1)\sqrt{\prime(d)}}$ and $w_{i,1}>1$. This upper bound scales exponentially with dimension. Different choices for $w_{p,k}$ are available, like choosing $w_{i,k}=w^{\frac{(k-1)\sqrt{\prime(k)}}{d\sqrt{\prime(d)}}}_{i,1}$ with $w_{i,1}$ sufficiently small gives $\norm{w_i}_{\ell^1}\leq 2d$. Second, both proofs are proven for $d\in\N$, whereas activation functions are univariate functions. We use the higher dimensional case when discussing the spectral Barron spaces in \Cref{sec:taylor}. 

\Cref{lemma:series_repu_s} can be used to write the $(x-t)^s$ in the remainder part of \eqref{eq:taylor_remainder} as a linear combination of $\repu_s$'s. However, this does not allow us to write the remainder part using a single measure $\nu\in\M(\R^2)$ for all $x\in\R$, because the integral bounds depend on $x$. In the second lemma, we deal with this by extending the domain of integration. 

\begin{lemma}\label{lemma:integral_bound_fix_lemma_appendix}
Let $s\geq 0$ and $c>0$. When $f\in L^1([-c,c])$, we have
\begin{equation}\label{eq:integral_bound_fix_lemma_appendix}
    \int_0^z f(u)(z-u)^sdu = \int_0^c f(u)\repu_s(z-u)+(-1)^{s-1}f(-u)\repu_s(-z-u)du
\end{equation}
for all $z\in [-c,c]$.
\end{lemma}
\begin{proof}
Depending on the sign of $z$ we can write the left-hand side equivalently as
\begin{equation}
    \int_0^z(z-u)^{s}f(u)du = \begin{cases}
        (-1)^{s-1}\int_0^c (-z-u)^{s}\step(-z-u)f(-u)du & -c \leq z\leq 0 \\
        \phantom{(-1)^{s-1}}\int_0^c (z-u)^{s}\step(z-u)f(u)du & 0 \leq z\leq c
    \end{cases}
\end{equation}
where $\step$ is the Heaviside step function. Note that the term $(-1)^{s-1}$ restores the sign for even $s$. Since both representations are zero in the domain of the other, we can add them to obtain
\begin{equation}\label{eq:integral_bound_fix_lemma_eq_with_indicator}
    \int_0^z(z-u)^{s}f(u)du = \int_0^c (z-u)^{s}\step(z-u)f(u)+(-1)^{s-1}(-z-u)^{s}\step(-z-u)f(-u)du.
\end{equation}
Observe that
\begin{equation}\label{eq:repu_step_with_poly}
\begin{aligned}
    (z-u)^{s}\step(z-u) &= \repu_{s}(z-u), \\
    (-z-u)^{s}\step(-z-u) &= \repu_{s}(-z-u).
\end{aligned}
\end{equation}
Substitution of \eqref{eq:repu_step_with_poly} into \eqref{eq:integral_bound_fix_lemma_eq_with_indicator} gives \eqref{eq:integral_bound_fix_lemma_appendix}.
\end{proof}

\end{document}